\documentclass[letterpaper]{article} 
\usepackage[preprint]{aaai2027}  
\usepackage[hyphens]{url}  
\usepackage{graphicx} 
\usepackage{natbib}  
\usepackage{caption} 
\usepackage{booktabs}

\usepackage[switch]{lineno}

\usepackage[linesnumbered,ruled,vlined]{algorithm2e}
\SetAlgoNlRelativeSize{0}
\usepackage{algpseudocode}
\usepackage{array}
\usepackage{amsmath}
\usepackage{amssymb}
\usepackage{amsthm}
\usepackage{tikz}

\newtheorem{definition}{Definition}
\newtheorem{theorem}{Theorem}
\newtheorem{lemma}{Lemma}

\newcommand{\C}{\mathcal{C}}
\newcommand{\CM}{\mathcal{C}_M}

\title{Key-Interval A*: Accelerating Grid Pathfinding via Structural Abstraction}

\author{
    Taiquan Sui
}

\affiliations{
    Department of Computer Science and Engineering\\
    Chalmers University of Technology\\
    Gothenburg, Sweden\\
    taiquan@chalmers.se
}

\begin{document}

\maketitle

\begin{abstract}

Existing exact methods for 4-connected grid pathfinding reduce online search, but often either retain fine-grained search states or require substantial preprocessing. This paper presents Key-Interval A* (KIA*), an optimal pathfinding algorithm that uses lightweight preprocessing to construct and search over a compact interval-level abstraction of free space. KIA* represents free space using intervals: maximal contiguous runs of traversable cells. It extracts key intervals that capture structural boundary changes and connects them through contiguous non-key regions. KIA* then performs A*-style search on the resulting key-interval graph and constructively reconstructs grid paths from interval chains, without cell-level local search. We prove the completeness and optimality of KIA* on 4-connected grids. Experiments on standard benchmarks show that KIA* preserves exact shortest-path lengths and achieves the fastest runtime on seven of eight benchmark groups, with the largest gains on structured and game maps.

\end{abstract}

\section{Introduction}

Pathfinding on grid maps is a fundamental problem in game AI, robotics, and autonomous navigation. Classical search algorithms such as A*~\cite{DBLP:journals/tssc/HartNR68} guarantee completeness and shortest-path optimality, but their efficiency degrades as map size and query difficulty increase. A key reason is that these algorithms operate at the level of individual grid cells, treating each cell as an independent search state. In large or structured environments, this representation fails to exploit the geometric continuity of free space, leading to extensive exploration of states that are locally distinct but structurally equivalent. As a result, optimal grid search can incur substantial computational overhead even when the underlying free-space structure is simple and highly regular.

Existing techniques accelerate grid-based pathfinding in several ways.
Symmetry-based methods such as Jump Point Search~\cite{DBLP:conf/aaai/HaraborG11}
and rectangular decomposition~\cite{DBLP:conf/aiide/HaraborB10}
prune redundant expansions, but still operate over grid-level states.
Other methods introduce higher-level structures: Subgoal
graphs~\cite{DBLP:conf/aips/UrasKH13} abstract obstacle-corner decision
points, HPA*~\cite{DBLP:journals/jogd/Botea0S04} partitions the environment into fixed-size clusters and relies on the resulting spatial abstraction to guide high-level search, and database-based approaches such as CPD~\cite{DBLP:conf/aiide/Botea11}
and Block A*~\cite{DBLP:conf/aaai/YapBHS11} further trade runtime efficiency for heavy offline computation and memory requirements. Despite these advances, the search representations used by existing methods are typically derived from cell-level pruning, sparse decision points, fixed hierarchical partitions, or precomputed path information. 

In this work, we introduce Key-Interval A* (KIA*), an optimal pathfinding algorithm for 4-connected grids based on a structural interval abstraction of free space, where intervals are maximal contiguous runs of traversable cells along grid rows or columns. The main contributions are: (1) a lightweight preprocessing procedure that extracts key intervals from local boundary changes and connects them through contiguous non-key regions; (2) an A*-style search algorithm over the resulting key-interval graph, with waypoint-update and dominance rules that preserve completeness and shortest-path optimality; and (3) an evaluation on standard grid benchmarks showing that KIA* preserves exact shortest-path lengths and achieves strong online runtime performance, particularly on structured and game maps.

\section{Preliminaries and Notation}

We consider a finite two-dimensional 4-connected grid graph
$G=(V,E)$ of size $M\times N$. Each grid cell $(x,y)$ represents a
vertex, where $x$ increases from left to right and $y$ increases from
top to bottom. Cells outside the grid boundary are treated as non-traversable.
The set $V\subseteq\mathbb{Z}^2$ contains all traversable cells, and
two vertices $(x_1,y_1),(x_2,y_2)\in V$ are connected by an edge if and
only if $|x_1-x_2|+|y_1-y_2|=1$. Every edge has unit cost.

Given a start vertex $s\in V$ and a target vertex $t\in V$, the objective
is to find a path $\pi=\langle v_0,v_1,\ldots,v_k\rangle$ such that
$v_0=s$, $v_k=t$, and every consecutive pair is adjacent. The path
length is the number of edges, namely $|\pi|=k$.

We use dot notation for structured entities; for a vertex $v=(x,y)$,
$v.x$ and $v.y$ denote its coordinates. For vertices $a,b\in V$, let
$\mathcal{C}(a,b)$ denote their shortest-path cost in $G$ if such a path exists. Let
\[
\mathcal{C}_M(a,b)=|a.x-b.x|+|a.y-b.y|
\]
denote the Manhattan distance between $a$ and $b$.

\begin{definition}[Manhattan Feasibility]
Two vertices $a,b\in V$ are said to be \emph{Manhattan-feasible} if there
exists a 4-connected path from $a$ to $b$ whose length equals their
Manhattan distance $\mathcal{C}_M(a,b)$.
Such a path is called a \emph{Manhattan-feasible path}.
Equivalently, $\mathcal{C}(a,b)=\mathcal{C}_M(a,b)$.
\end{definition}

\section{Preprocessing}

To enable efficient pathfinding, we conduct a one-time preprocessing phase on the input grid map. This phase identifies intervals, extracts critical structural change points (termed \emph{key points}), and designates the intervals containing these key points as \emph{Key Intervals} for search. These components form the structural backbone of the Key-Interval A* algorithm.

We first introduce the notion of an \emph{interval}, which serves as the basic unit of
abstraction throughout the algorithm.

\begin{definition}[Interval]
A vertical \emph{interval} is a maximal contiguous sequence of traversable grid cells
along a single column. Formally, a vertical interval is represented as
\[
I = \langle x, y_s, y_e \rangle,
\]
where every cell $(x,y)$ with $y\in[y_s,y_e]$ is traversable, and both
$(x,y_s-1)$ and $(x,y_e+1)$ are non-traversable.
The endpoints $(x,y_s)$ and $(x,y_e)$ are referred to as the start and end vertices of $I$.
A horizontal interval is defined analogously as a maximal contiguous sequence of
traversable cells along a single row and is represented as $\langle y, x_s, x_e \rangle$.
\end{definition}

Figure~\ref{fig:preprocess} illustrates an example. On column $x=0$ there is a single
vertical interval $\langle 0,0,7\rangle$, whereas column $x=3$ contains two vertical
intervals, $\langle 3,0,2\rangle$ and $\langle 3,5,7\rangle$.

\begin{definition}[Direct Neighbor]
For a vertical interval $I=\langle x,y_s,y_e\rangle$, 
a vertical interval $I'=\langle x',y'_s,y'_e\rangle$ is a \emph{direct neighbor} of $I$ if
$|x' - x| = 1$ and 
$[y_s,y_e]\cap[y'_s,y'_e]\neq\emptyset$.
Two intervals are said to be \emph{connected} if they are direct neighbors of each other.
\end{definition}

The set of direct neighbors of $I$ is denoted by $N_d(I)$.
A direct neighbor $I'$ is called a \emph{left} (resp.\ \emph{right}) direct
neighbor if $x'=x-1$ (resp.\ $x'=x+1$), denoted by
$I'\in N_d^l(I)$ (resp.\ $I'\in N_d^r(I)$).

As shown in Figure~\ref{fig:preprocess}, the interval $\langle 2,0,7\rangle$ has a left direct neighbor $\langle 1,1,6\rangle$ and right direct neighbors $\langle 3,0,2\rangle$ and $\langle 3,5,7\rangle$, which together constitute its direct neighbor set.

\begin{definition}[Horizontally Monotone Interval Chain]
Let
\[
\Gamma=(I_0,\ldots,I_m),
\qquad m\ge 1,
\]
be an ordered sequence of vertical intervals, where
$I_i=\langle x^{(i)},y_s^{(i)},y_e^{(i)}\rangle$.
We call $\Gamma$ a \emph{horizontally monotone interval chain} if
\[
I_{i+1}\in N_d(I_i)
\qquad
\text{for all }0\le i<m,
\]
and there exists $d\in\{-1,+1\}$ such that
\[
x^{(i+1)}=x^{(i)}+d
\qquad
\text{for all }0\le i<m.
\]

For such a chain, we define
\[
\operatorname{exit}(\Gamma)=I_1,
\qquad
\operatorname{entry}(\Gamma)=I_{m-1}.
\]
Here, $\operatorname{exit}(\Gamma)$ is the direct neighbor
immediately after leaving $I_0$, whereas
$\operatorname{entry}(\Gamma)$ is the direct neighbor from which
$I_m$ is entered. When $m=1$, these definitions give
$\operatorname{exit}(\Gamma)=I_m$ and
$\operatorname{entry}(\Gamma)=I_0$.
\end{definition}

\subsection{Key Point Extraction}

We identify key points as vertices that capture local structural changes
in the shape of free space. Intuitively, as we scan the grid column by
column, the boundaries of traversable regions may expand, shrink, or
remain aligned. Key points correspond to locations where these boundary
trends change in a way that disrupts direct Manhattan-feasible
connectivity.

\begin{figure}
    \centering
    \begin{tikzpicture}[scale=0.45]

    \fill[gray!70] (1, 0) rectangle (2, 1);
    \fill[gray!70] (1, 7) rectangle (2, 8);

    \fill[gray!70] (3, 3) rectangle (5, 5);

    \fill[gray!70] (7, 2) rectangle (8, 8);

    \fill[gray!70] (9, 0) rectangle (10, 5);
    \fill[gray!70] (10, 4) rectangle (13, 5);
    \fill[gray!70] (11, 6) rectangle (14, 7);
    \fill[gray!70] (11, 7) rectangle (12, 8);
    \fill[gray!70] (13, 1) rectangle (14, 5);
    \fill[gray!70] (12, 1) rectangle (13, 2);
    \fill[gray!70] (11, 1) rectangle (12, 3);

    \draw[thick] (0, 0) grid (15, 8);

    \foreach \x in {0,...,14} {
        \node[above] at (\x+0.5,8) {\x};
    }

    \foreach \y [count=\yy from 0] in {7,...,0} {
        \node[left] at (0,\yy+0.5) {\y};
    }

    \draw[->] (1.5,1.3) -- (1.5,1.7);
    \draw[->] (1.5,6.7) -- (1.5,6.3);

    \draw[->] (2.7,3.5) -- (2.3,3.5);

    \draw[->] (4.5,2.7) -- (4.5,2.3);
    \draw[->] (4.5,5.3) -- (4.5,5.7);

    \draw[->] (5.3,3.5) -- (5.7,3.5);

    \draw[->] (7.5,1.7) -- (7.5,1.3);

    \draw[->] (10.7,1.5) -- (10.3,1.5);

    \draw[->] (11.5,3.3) -- (11.5,3.7);

    \draw[<->] (13.5,5.2) -- (13.5,5.8);
    \draw[->] (13.5,0.7) -- (13.5,0.3);
    
    \draw[->] (14.3,6.5) -- (14.7,6.5);
    \draw[->] (14.3,1.5) -- (14.7,1.5);

    \end{tikzpicture}
    \caption{Example of preprocessing. Vertices marked with arrows indicate key points along with their associated directions}
    \label{fig:preprocess}
\end{figure}
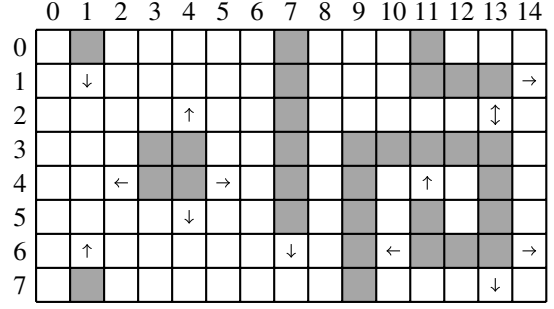

\paragraph{Trend Computation}
We perform a vertical scan from left to right. For each vertical interval
$I=\langle x,y_s,y_e\rangle$, we compare its endpoints with those of its
direct neighbors in the previous column.

If $N^l_d(I)=\emptyset$, then $I$ has no left direct neighbor and no
boundary comparison is available. In this case, both boundary trends are
initialized as
\[
T_s(I)=T_e(I)=\textsc{LEVEL}.
\]

Otherwise, let
\[
I^l_{\min}(I) := \arg\min_{I'\in N^l_d(I)} I'.y_s,
\]
\[
I^l_{\max}(I) := \arg\max_{I'\in N^l_d(I)} I'.y_e,
\]
denote the left neighbor with the smallest start coordinate
and the left neighbor with the largest end coordinate. The right-side quantities $I^r_{\min}(I)$ and $I^r_{\max}(I)$ are defined
analogously over $N_d^r(I)$.

To capture boundary evolution, we define two difference values:
\[
\Delta_s(I) := I.y_s - I^l_{\min}(I).y_s,
\quad
\Delta_e(I) := I.y_e - I^l_{\max}(I).y_e.
\]

Based on these differences, the start trend $T_s(I)$ and end trend
$T_e(I)$ are defined as
\[
T_s(I)=
\begin{cases}
\textsc{DOWN}  & \text{if } \Delta_s(I) > 0,\\
\textsc{UP}    & \text{if } \Delta_s(I) < 0,\\
T_s(I^l_{\min}(I)) & \text{if } \Delta_s(I)=0,
\end{cases}
\]
and
\[
T_e(I)=
\begin{cases}
\textsc{DOWN}  & \text{if } \Delta_e(I) > 0,\\
\textsc{UP}    & \text{if } \Delta_e(I) < 0,\\
T_e(I^l_{\max}(I)) & \text{if } \Delta_e(I)=0.
\end{cases}
\]

These trends record the direction of the boundary shift, inheriting the previous trend if the boundary remains aligned.

For example, consider the interval $\langle 3, 0, 2 \rangle$ in Figure~\ref{fig:preprocess}. Its $I^l_{\max}(I)$ is $\langle 2, 0, 7 \rangle$. Here, $\Delta_e = 2 - 7 = -5 < 0$, strictly shifting the end boundary upward, yielding $T_e = \textsc{UP}$.

\paragraph{Key Point Identification}

Based on this trend information, we extract \emph{key points}—vertices that represent local structural inflections where free space contracts and subsequently expands. 
For a vertical interval $I=\langle x,y_s,y_e\rangle$, if $N_d^r(I)\neq\emptyset$, its endpoints are labeled according to the
following rules:
\begin{itemize}
\item The start vertex $(x,y_s)$ is labeled $\downarrow$ if
\[
T_s(I)=\textsc{DOWN}
\quad\text{and}\quad
I^r_{\min}(I).y_s < I.y_s.
\]
\item The end vertex $(x,y_e)$ is labeled $\uparrow$ if
\[
T_e(I)=\textsc{UP}
\quad\text{and}\quad
I^r_{\max}(I).y_e > I.y_e.
\]
\end{itemize}
If $N_d^r(I)=\emptyset$, neither endpoint is labeled.

The horizontal scan is defined symmetrically by exchanging the roles of $x$ and $y$. In this case, the trends \textsc{UP} and \textsc{DOWN} correspond to \textsc{LEFT} and \textsc{RIGHT}, respectively, yielding $\leftarrow$ and $\rightarrow$ key points. An endpoint is defined as a \emph{key point} if and only if it is assigned
one of the above directional labels. 
Figure~\ref{fig:preprocess} illustrates the detected key points.

Vertical intervals serve as the abstraction units used by KIA*, while
horizontal intervals are used only in the symmetric key-point extraction
pass. KIA* can also be formulated with horizontal intervals as the
abstraction units by symmetrically exchanging the roles of the two
orientations. This paper considers only the
vertical-interval formulation.

\subsection{Key Interval}

To construct a compact graph representation, we select vertical intervals that are anchored by structural features, termed \emph{key intervals}.
Formally, any vertical interval $I$ containing at least one key point is designated as a \emph{key interval}, denoted by $K$. Let $\mathcal{K}$ denote the set of all key intervals.
Since a vertical interval has a unique start vertex and end vertex, $K$ contains at most one $\uparrow$ key point and at most one $\downarrow$ key point, denoted by $K.\text{up}$ and $K.\text{down}$, respectively. 

\begin{theorem}
\label{thm:key-point}
Let $I$ be a vertical interval.
If $|N^l_d(I)| > 1$ or $|N^r_d(I)| > 1$,
then $I$ contains at least one key point and is therefore a key interval.
\end{theorem}

\begin{proof}
We prove the case $|N^l_d(I)|>1$; the case $|N^r_d(I)| > 1$ is symmetric.

Let
$I=\langle x+1,y_s,y_e\rangle$
be a vertical interval whose left direct neighbors are
\[
I_j=\langle x,y_s^{(j)},y_e^{(j)}\rangle,\qquad j=1,\ldots,n,
\]
where 
$y_s^{(1)}\le y_e^{(1)}
<
y_s^{(2)}\le y_e^{(2)}
<
\cdots
<
y_s^{(n)}\le y_e^{(n)}$.

Consider two consecutive direct neighbors $I_k$ and $I_{k+1}$. All cells in
$[y_e^{(k)}+1,\; y_s^{(k+1)}-1]$ of column $x$ are blocked.

We analyze the horizontal scan on rows
\[
y^{(i)} := y_e^{(k)} + i,
\qquad
0 \le i \le y_s^{(k+1)} - y_e^{(k)}.
\]
Let
$H^{(i)}=\langle y^{(i)},x_s^{(i)},x_e^{(i)}\rangle$
denote the horizontal interval containing $(x+1,y^{(i)})$, which exists since $(x+1, y^{(i)}) \in I$.

At $i=0$, the presence of $I_k$ implies $x_s^{(0)} \le x$.

At $i=1$, column $x$ becomes blocked while column $x+1$ remains
traversable, so $x_s^{(1)} = x+1$. Hence
$\Delta_s(H^{(1)}) > 0$ and therefore
$T_s(H^{(1)})=\textsc{RIGHT}$.

For
$1 < i < y_s^{(k+1)} - y_e^{(k)}$,
the same configuration holds, so
$x_s^{(i)} = x+1$ and $\Delta_s(H^{(i)}) = 0$.
Thus the trend is inherited, yielding
$T_s(H^{(i)})=\textsc{RIGHT}$.

Finally, at
$i = y_s^{(k+1)} - y_e^{(k)}$,
the interval $I_{k+1}$ begins and column $x$ becomes traversable again.
Hence
$x_s^{(i)} \le x < x_s^{(i-1)} = x+1$,
which implies $\Delta_s(H^{(i)}) < 0$.
Since $T_s(H^{(i-1)})=\textsc{RIGHT}$ and
$x_s^{(i)} < x_s^{(i-1)}$, the horizontal key-point rule labels the
start vertex of $H^{(i-1)}$ as a $\rightarrow$ key point. Since
$H^{(i-1)}.x_s = x+1$ and
$y^{(i-1)} = y_s^{(k+1)}-1$, this key point is exactly
$(x+1,y_s^{(k+1)}-1)$.
Since this vertex lies in $I$, the interval $I$ contains a key point
and is therefore a key interval.
\end{proof}

\begin{definition}[Key-Interval Graph]
The \emph{key-interval graph} is the directed multigraph
$\mathcal{G}_{\mathcal{K}}
=(\mathcal{V}_{\mathcal{K}},\mathcal{E}_{\mathcal{K}})$, where
$\mathcal{V}_{\mathcal{K}}=\mathcal{K}$.
Each edge
$\Gamma=(I_0,\ldots,I_m)\in\mathcal{E}_{\mathcal{K}}$
is a horizontally monotone interval chain directed from the key interval
$I_0$ to the key interval $I_m$, with all intermediate intervals, if
any, being non-key. Distinct chains are retained as distinct edges even
if they have the same first and last intervals.

For each key interval $K$, let
\[
\mathcal{E}_{\mathrm{out}}(K)
=
\left\{
\Gamma=(I_0,\ldots,I_m)\in\mathcal{E}_{\mathcal{K}}
\mid I_0=K
\right\}
\]
denote its outgoing edges.
\end{definition}

To construct $\mathcal{E}_{\mathcal{K}}$, for each key interval $K$
and each direct neighbor $I_d\in N_d(K)$, we initialize
\[
\Gamma=(K,I_d),
\qquad
I_{\mathrm{curr}}=I_d,
\qquad
d=\operatorname{sign}(I_d.x-K.x).
\]
If $I_{\mathrm{curr}}$ is key, $\Gamma$ is inserted into
$\mathcal{E}_{\mathcal{K}}$ and the search terminates. Otherwise, the
search follows the unique direct neighbor $I_{\mathrm{next}}$, if it
exists, satisfying
\[
I_{\mathrm{next}}.x-I_{\mathrm{curr}}.x=d.
\]
Whenever such an interval exists, it is appended to $\Gamma$ and becomes
the new $I_{\mathrm{curr}}$. The search terminates when a key interval is
reached, in which case $\Gamma$ is inserted into
$\mathcal{E}_{\mathcal{K}}$, or when no continuation exists, in which
case no edge is added.

By the contrapositive of Theorem~\ref{thm:key-point}, every non-key
interval has at most one direct neighbor on its left and at most one on
its right. Hence, $I_{\mathrm{next}}$ is unique whenever it exists.
Consequently, the non-key intervals form disjoint non-branching chains
under the direct-neighbor relation. We call each maximal such chain a
\emph{non-key component}.

The non-key intervals traversed by each directional search are assigned
to the same component. After all searches have been processed, a final
linear pass assigns any remaining non-key interval to its maximal
component. Such remaining components are not
reachable from the key-interval graph.

To interface non-key components with the key-interval graph used for
search, we define their boundary key intervals. For each non-key
component $C$, let
\[
\mathcal{B}(C)
=
\{K\in\mathcal{K}\mid \exists I\in C,\ K\in N_d(I)\}.
\]
These intervals serve as interfaces between $C$ and the
key-interval graph. We write
\[
\overline{C}=C\cup\mathcal{B}(C)
\]
for the closure of $C$. A vertex $v\in V$ is said to lie in
$\overline{C}$ if the vertical interval containing $v$ belongs to
$\overline{C}$.

For any two distinct intervals $I_a,I_b\in\overline{C}$, let
$\Gamma_C(I_a,I_b)$
denote the unique horizontally monotone interval chain contained in
$\overline{C}$ that begins at $I_a$, ends at $I_b$, and has all
intermediate intervals, if any, in $C$. Its uniqueness follows from the
non-branching structure of $C$.

\begin{definition}[Non-Key Context]
\label{def:context}
Let $s, t\in V$ be two vertices. We say that $s$ and $t$ share
a \emph{non-key context} if there exists a non-key component $C$ such
that
$s\in\overline{C}$
and
$t\in\overline{C}$.

\end{definition}

\begin{lemma}[Manhattan Feasibility in Non-Key Context]
\label{lem:context-manhattan}
Let $s,t \in V$ be two vertices. If $s$ and $t$ share a non-key context, there exists a Manhattan-feasible path from $s$
to $t$.
\end{lemma}

\begin{proof}
Let $I_s$ and $I_t$ be the vertical intervals containing $s$ and $t$.
If $I_s=I_t$, the vertical path between them is Manhattan-feasible.
Otherwise, by the definition of non-key context, there exists a non-key
component $C$ such that $s,t\in\overline{C}$. Let
\[
\Gamma_C(I_s,I_t)=(I_0,\ldots,I_m)
\]
be the horizontally monotone interval chain from $I_s$ to $I_t$ inside
$\overline{C}$.
Assume by symmetry that it proceeds to the right.
Write
$I_i=\langle x^{(i)},y_s^{(i)},y_e^{(i)}\rangle$,
and prove the case $s.y\le t.y$; the case $s.y>t.y$ is symmetric.

We first establish a boundary claim. Consider any subchain
$I_p,\ldots,I_q$ whose intermediate intervals are non-key. For any row
level $\eta$, if
$y_s^{(p)}\le \eta$ and $y_s^{(q)}\le \eta$, then
$y_s^{(i)}\le \eta$ for all $p\le i\le q$.
Indeed, suppose otherwise. Let $j$ be the first index with
$y_s^{(j)}>\eta$, and let $k\ge j$ be the first index such that
$y_s^{(k+1)}<y_s^{(k)}$; such $k$ exists because the boundary later
returns to at most $\eta$. Since $k$ is the first index after $j$ with
$y_s^{(k+1)}<y_s^{(k)}$, the lower boundary does not decrease from
$I_j$ to $I_k$; each step either sets the start trend to
$\textsc{DOWN}$ again or inherits it. Thus $T_s(I_k)=\textsc{DOWN}$.
Since its right direct neighbor $I_{k+1}$ satisfies
$y_s^{(k+1)}<y_s^{(k)}$, the key-point rule would label a
$\downarrow$ key point in $I_k$, contradicting that $I_k$ is non-key.
Thus the lower boundary cannot rise above level $\eta$ and later return.

Symmetrically, if $y_e^{(p)}\ge \eta$ and $y_e^{(q)}\ge \eta$, then
$y_e^{(i)}\ge \eta$ for all $p\le i\le q$; otherwise an $\uparrow$ key
point would be induced.

Applying the claim to the whole chain with level $t.y$ for the lower
boundary and level $s.y$ for the upper boundary gives
$y_s^{(i)}\le t.y$ and $y_e^{(i)}\ge s.y$ for every $i$. For each
adjacent pair $I_i,I_{i+1}$, let
\[
O_i=[\ell_i,u_i]
=
[\max(y_s^{(i)},y_s^{(i+1)}),
 \min(y_e^{(i)},y_e^{(i+1)})]
\]
be their row-overlap interval. Since the intervals are direct neighbors,
$O_i$ is nonempty, and the boundary claim above implies $\ell_i\le t.y$ and
$u_i\ge s.y$.

Set $r_0=s.y$ and define
\[
r_{i+1}=\max\{r_i,\ell_i\}
\]
for $0\le i<m$. We show that each $r_{i+1}$ is a valid crossing row
between $I_i$ and $I_{i+1}$, i.e., $r_{i+1}\in O_i$.

First, since $r_0=s.y$ and $\ell_i\le t.y$ for all $i$, an immediate
induction gives
\[
s.y\le r_i\le t.y
\qquad\text{for all } i.
\]
Moreover, by definition, $r_{i+1}\ge \ell_i$. It remains to show that
$r_{i+1}\le u_i$.

Suppose, for contradiction, that $r_{i+1}>u_i$. Since
$r_{i+1}=\max\{r_i,\ell_i\}$ and $\ell_i\le u_i$, we have $r_i>u_i$.
Because
$r_i$ is a valid row of $I_i$, we have $r_i\le y_e^{(i)}$. Since
$u_i=\min\{y_e^{(i)},y_e^{(i+1)}\}$, the inequalities
$r_i>u_i$ and $r_i\le y_e^{(i)}$ imply $u_i=y_e^{(i+1)}$, and hence
\[
y_e^{(i+1)}<r_i\le y_e^{(i)}.
\]
Moreover, since $r_i\le t.y$ and $t\in I_m$, we have
$y_e^{(m)}\ge t.y\ge r_i$. Applying the boundary
claim to the subchain $I_i,\ldots,I_m$ with level $\eta=r_i$ gives
$y_e^{(i+1)}\ge r_i$, contradicting $y_e^{(i+1)}<r_i$. Hence
$r_{i+1}\le u_i$, and therefore $r_{i+1}\in O_i$.

We can therefore construct a path from $s$ to $t$ as follows. For
each $0\le i<m$, the path first moves vertically inside $I_i$ from row
$r_i$ to row $r_{i+1}$, and then moves horizontally to $I_{i+1}$ at row
$r_{i+1}$. After reaching $I_m$, it moves vertically inside $I_m$ from
row $r_m$ to $t.y$. All horizontal moves have the same direction, and the row sequence
is nondecreasing. Thus the horizontal cost is $|t.x-s.x|$, the vertical
cost is $|t.y-s.y|$, and the total length is $\CM(s,t)$. Hence the
constructed path is Manhattan-feasible.
\end{proof}

\begin{definition}[Transition Vertex]
Let $K=\langle x,y_s,y_e\rangle$ be a key interval, and let
$\sigma\in\{l,r\}$ denote one side of $K$. Let
$I_1,\ldots,I_q$ be the intervals in $N_d^\sigma(K)$ on one side of $K$,
sorted by start coordinate. For two distinct same-side neighbors
$I_\alpha$ and $I_\beta$, assume $\alpha<\beta$. A key point $p=(x,y)$
of $K$ is a \emph{transition vertex} between them if, for some
$\alpha\le i<\beta$,
\[
I_i.y_e<y<I_{i+1}.y_s .
\]
\end{definition}

Each separating gap contains a key point by
Theorem~\ref{thm:key-point}. We store one representative key point for
each gap between consecutive same-side neighbors. For
$I_\alpha$ and $I_\beta$ with $\alpha<\beta$,
$K.T(I_\alpha,I_\beta)$ returns the representative of the
lowest-index separating gap, namely the gap between
$I_\alpha$ and $I_{\alpha+1}$.

As shown in Figure~\ref{fig:preprocess}, the vertex $(2,4)$ is a
transition vertex for the gap between $\langle 3,0,2\rangle$ and
$\langle 3,5,7\rangle$ via the key interval $\langle 2,0,7\rangle$.
For the gap between $\langle 13,0,0\rangle$ and $\langle 13,7,7\rangle$
via $\langle 14,0,7\rangle$, $(14,1)$ is selected
as the representative.

\section{Key-Interval A*}

Key-Interval A* (KIA*) performs A*-style search over the precomputed
key-interval graph $\mathcal{G}_{\mathcal{K}}$.

\subsection{Algorithm Description}

Each nonterminal search state $c$ is associated with a key interval
$c.K$ and represents an arrival-cost profile over that interval. It
stores cost values $g(c)$ and $h(c)$, an ordered waypoint sequence
$W(c)$, and an incoming interval
\[
c.N_{\mathrm{in}}\in N_d(c.K)\cup\{\bot\}.
\]
The search priority of every state is
$f(c)=g(c)+h(c)$.
A terminal state has $c.K=c.N_{\mathrm{in}}=\bot$ and a waypoint
sequence ending at $t$.

The waypoint sequence records only the vertices required to preserve
Manhattan-feasible connectivity. For a nonterminal state, if
$c.N_{\mathrm{in}}=\bot$, the state originates inside $c.K$;
otherwise, $c.N_{\mathrm{in}}$ is the direct neighbor through which
$c.K$ is entered. This incoming interval is used during expansion to
determine whether a transition vertex is required.

For path reconstruction, every noninitial state stores a parent
pointer. Each nonterminal successor also stores the graph edge used to
reach it. An initial state stores its initialization chain when
$I_s\notin\mathcal{K}$, and a terminal state stores the chain from the
last expanded key interval to $I_t$.

\paragraph{Initialization}
Let $I_s$ and $I_t$ denote the vertical intervals containing $s$ and
$t$, respectively. If $s$ and $t$ share a non-key context, KIA*
constructs a Manhattan-feasible path between them by
Lemma~\ref{lem:context-manhattan} and terminates.

Otherwise, for each $q\in\{s,t\}$, set
$\mathcal{K}_q=\{I_q\}$ if $I_q\in\mathcal{K}$; otherwise, let $C_q$
be the non-key component containing $I_q$ and set
$\mathcal{K}_q=\mathcal{B}(C_q)$.
The sets $\mathcal{K}_s$ and $\mathcal{K}_t$ identify the key intervals
used to initialize the search and those from which the final connection
to $t$ is attempted, respectively.

If $I_s\in\mathcal{K}$, an initial state $c$ with
$c.K=I_s$, $c.N_{\mathrm{in}}=\bot$, and
$W(c)=\langle s\rangle$ is inserted into the open list. Otherwise, for
each $K\in\mathcal{K}_s$, let
$\Gamma_s=\Gamma_{C_s}(I_s,K)$. An initial state $c$ with
$c.K=K$,
$c.N_{\mathrm{in}}=\operatorname{entry}(\Gamma_s)$, and
$W(c)=\langle s\rangle$ is inserted into the open list. By
Lemma~\ref{lem:context-manhattan}, $s$ is Manhattan-feasible with every
vertex of $K$, so each initial state represents a valid arrival-cost
profile over its key interval.

\paragraph{Search}
At each iteration, the state with the smallest $f$-value is extracted
from the open list. If its last waypoint is $t$, the path reconstructed
from its waypoint sequence as described later is returned. Otherwise, the state is expanded
as described in the next section. Successor states record their incoming
intervals, update their waypoint sequences to preserve
Manhattan-feasible connectivity, and are inserted into the open list. The
search terminates when a terminal state ending at $t$ is extracted, or
when the open list becomes empty.

\subsection{Node Expansion}

Consider the expansion of a state $c$ associated with the key interval
$K=c.K$.

If $K\notin\mathcal{K}_t$, then for each
$\Gamma=(I_0,\ldots,I_m)\in\mathcal{E}_{\mathrm{out}}(K)$, KIA*
generates a candidate successor $c'$ with $c'.K=I_m$,
$c'.N_{\mathrm{in}}=\operatorname{entry}(\Gamma)$, and
$N_{\mathrm{out}}=\operatorname{exit}(\Gamma)$.

If $K\in\mathcal{K}_t$, KIA* instead generates a candidate terminal
successor $c'$ with $c'.K=c'.N_{\mathrm{in}}=\bot$. If $I_t=K$, set $N_{\mathrm{out}}=\bot$. Otherwise, $t$ lies
in a non-key component $C_t$, and set
$N_{\mathrm{out}}=\operatorname{exit}(\Gamma_{C_t}(K,I_t))$.

In either case, the candidate is discarded if
$c.N_{\mathrm{in}}\neq\bot$ and
$N_{\mathrm{out}}=c.N_{\mathrm{in}}$. Such a candidate leaves $K$
through the same direct neighbor from which $K$ was entered. The detour
crosses between $c.N_{\mathrm{in}}$ and $K$ twice and can be replaced by
the vertical segment between the same crossing rows within
$c.N_{\mathrm{in}}$, strictly reducing its length.

Any remaining candidate initializes its waypoint sequence as
$W(c')=W(c)$. Let $w_n$ be the last waypoint in this sequence.
The sequence $W(c')$ is then updated according to the following priority
rules:

\begin{itemize}
\item If $K.T(c.N_{\mathrm{in}},N_{\mathrm{out}})$ is defined, append it to $W(c')$.

\item Otherwise, if $K$ contains an $\uparrow$ key point $K.\text{up}$
and $w_n.y>K.\text{up}.y$, append $K.\text{up}$ to $W(c')$.

\item Otherwise, if $K$ contains a $\downarrow$ key point
$K.\text{down}$ and $w_n.y<K.\text{down}.y$, append
$K.\text{down}$ to $W(c')$.
\end{itemize}
Finally, if $K\in\mathcal{K}_t$, append $t$ to $W(c')$. For every $K\in\mathcal{K}_t$, every vertex of $K$ is
Manhattan-feasible with $t$, either because $K=I_t$ or by
Lemma~\ref{lem:context-manhattan}. Hence leaving $K$ through another
graph edge cannot yield a shorter completion.

The values $g(c')$, $h(c')$, and $f(c')$ are computed according to the
state-evaluation rules described later. A terminal successor is inserted
directly into the open list, whereas a nonterminal successor is inserted
unless it is dominated.

The following lemmas justify the waypoint update rules.
Lemma~\ref{lem:waypoint-feasibility} shows that each appended waypoint
is Manhattan-feasible with the preceding waypoint, while
Lemma~\ref{lem:pruning-safety} shows that inserting the selected key
point or transition vertex does not eliminate an optimal continuation.

\begin{lemma}[Waypoint Feasibility]
\label{lem:waypoint-feasibility}
Suppose an expansion appends a waypoint $w$, where $w$ is
either a key point of the expanded key interval or the target vertex
$t$, and let $w_n$ be the preceding waypoint. Then $w_n$ and
$w$ are Manhattan-feasible, and hence
$\C(w_n,w)=\CM(w_n,w)$.
\end{lemma}
\begin{proof}
If $w_n$ and $w$ lie in the same vertical interval, the vertical
segment between them is Manhattan-feasible.

Otherwise, let
$\Gamma^\ast=(I_0,\ldots,I_m)$ be the interval sequence obtained by
concatenating, in traversal order, the interval chains followed since
$w_n$ was appended. If $w=t$, the sequence also includes the
terminal chain to $I_t$.

Since no waypoint is appended between $w_n$ and $w$,
$\Gamma^\ast$ is horizontally monotone. Indeed, any horizontal reversal
would enter and leave some key interval through direct neighbors on the
same side. Returning through the same neighbor is discarded during
expansion, whereas using distinct same-side neighbors would append a
transition vertex.

Assume by symmetry that $\Gamma^\ast$ proceeds to the right and that
$w_n.y\le w.y$. Write
$I_i=\langle x^{(i)},y_s^{(i)},y_e^{(i)}\rangle$. We claim:
\[
y_s^{(i)}\le w.y
\quad\text{and}\quad
y_e^{(i)}\ge w_n.y
\qquad\text{for all }i.
\]
If the lower boundary rose above $w.y$ and later returned, the
boundary argument in Lemma~\ref{lem:context-manhattan} would induce a
$\downarrow$ key point above $w_n.y$. The waypoint-update rules would
then append either that key point or a higher-priority transition
vertex, contradicting the choice of $w_n$ and $w$. The upper-bound
case is symmetric.

The claimed bounds are exactly the conditions used in the 
construction of Lemma~\ref{lem:context-manhattan}. Applying that
construction yields a path from $w_n$ to $w$ of length
$\CM(w_n,w)$. Hence
$\C(w_n,w)=\CM(w_n,w)$.
\end{proof}

\begin{lemma}[Pruning Safety]
\label{lem:pruning-safety}
Suppose the waypoint-update rules append a key point $w$ while
expanding a state $c$, and let $w_n$ be the waypoint immediately
preceding $w$.

If $w$ is a vertical key point, i.e.,
$w\in\{c.K.\text{up},c.K.\text{down}\}$, then for every vertex
$o\in c.K$ and every vertex $v$ reachable from $o$,
\[
\C(w_n,w)+\C(w,v)
\le
\C(w_n,o)+\C(o,v).
\]

If $w$ is a transition vertex for
$c.N_{\mathrm{in}}$ and $N_{\mathrm{out}}$, then any traversal of
$c.K$ that enters through $c.N_{\mathrm{in}}$ and leaves through
$N_{\mathrm{out}}$ can be replaced by one passing through $w$ without
increasing its length.
\end{lemma}

\begin{proof}
Let $K=c.K$. First suppose $w=K.\text{up}$; the case
$w=K.\text{down}$ is symmetric. By the update rule,
$w_n.y>w.y$. Since $w$ is the end vertex of $K$, every $o\in K$
satisfies $o.y\le w.y$, and therefore
\[
\CM(w_n,o)=\CM(w_n,w)+\CM(w,o).
\]
By Lemma~\ref{lem:waypoint-feasibility},
$\C(w_n,w)=\CM(w_n,w)$. Moreover, the vertical segment between $w$
and $o$ lies entirely in $K$, so $\C(w,o)=\CM(w,o)$. Hence, by the
triangle inequality,
\begin{align*}
\C(w_n,w)+\C(w,v)
&\le \C(w_n,w)+\C(w,o)+\C(o,v)\\
&= \CM(w_n,w)+\CM(w,o)+\C(o,v)\\
&= \CM(w_n,o)+\C(o,v)\\
&\le \C(w_n,o)+\C(o,v).
\end{align*}
Thus routing through $w$ is no more costly than routing through any
other vertex $o\in K$.

Now suppose that $w$ is a transition vertex. Let
$I_a=c.N_{\mathrm{in}}$ and $I_b=N_{\mathrm{out}}$. By definition,
$I_a$ and $I_b$ are distinct direct neighbors on the same side of $K$,
and $w$ lies in a transition gap separating them. Assume without loss
of generality that
\[
I_a.y_e<w.y<I_b.y_s.
\]

Consider a traversal that enters $K$ from $I_a$ at row $r_a$ and
leaves toward $I_b$ at row $r_b$. Since the crossing rows lie in the
corresponding overlaps,
\[
r_a\le I_a.y_e<w.y<I_b.y_s\le r_b.
\]
The vertical segment in $K$ from $(K.x,r_a)$ to $(K.x,r_b)$ is
traversable, passes through $w$, and has length
$r_b-r_a$, which is the Manhattan distance between its endpoints.
It therefore replaces the original traversal without increasing its
length. The entry and exit crossings remain unchanged, so the same
prefix and suffix paths remain available. Hence inserting $w$ is safe.
\end{proof}

\subsection{State Dominance and Cost Evaluation}
\label{sec:cost_dominance}

To preserve optimality, KIA* represents the arrival cost over an entire
key interval analytically. Let $c$ be a nonterminal search state
associated with a key interval $K=c.K$, and let
$W(c)=\langle w_1,\ldots,w_n\rangle$ be its waypoint sequence. We define
the projected arrival vertex on $K$ as
\[
v_{\mathrm{proj}}
=
\bigl(K.x,\max\{K.y_s,\min(w_n.y,K.y_e)\}\bigr).
\]
The projected vertex is used only for cost evaluation, dominance
pruning, and heuristic computation; it is not appended to the waypoint
sequence.

By the same reasoning as in
Lemma~\ref{lem:waypoint-feasibility}, $w_n$ is connected to
$v_{\mathrm{proj}}$ by a Manhattan-feasible path. We define the $g$-value of $c$,
which represents the cost of reaching $v_{\mathrm{proj}}$, as
\[
g(c)
=
\sum_{i=1}^{n-1}\CM(w_i,w_{i+1})
+
\CM(w_n,v_{\mathrm{proj}}).
\]
For any vertex $(K.x,y)$ on $K$, the projection vertex lies on a
Manhattan-shortest continuation from $w_n$ to $(K.x,y)$. Hence
\[
\CM(w_n,(K.x,y))
=
\CM(w_n,v_{\mathrm{proj}})
+
|y-v_{\mathrm{proj}}.y|.
\]
The cost of reaching $(K.x,y)$ through state $c$ is therefore
\[
\Phi_c(y)=g(c)+|y-v_{\mathrm{proj}}.y|.
\]

\paragraph{Dominance and Pruning}
Since multiple states may reach the same key interval, KIA* retains only
non-dominated arrival-cost profiles.

\begin{lemma}[State Dominance]
\label{lem:state-dominance}
Let $c_1$ and $c_2$ be two nonterminal states associated with the same
key interval $K$. Let $y_1$ and $y_2$ be their projected rows, and let
$g_1=g(c_1)$ and $g_2=g(c_2)$. If
\[
g_1+|y_1-y_2|\le g_2,
\]
then $c_1$ dominates $c_2$.
\end{lemma}

\begin{proof}
For any row $y\in[K.y_s,K.y_e]$, the triangle inequality gives
\[
|y-y_1|
\le
|y_1-y_2|+|y-y_2|.
\]
Therefore,
\begin{align*}
\Phi_{c_1}(y)
&=g_1+|y-y_1|\\
&\le g_1+|y_1-y_2|+|y-y_2|\\
&\le g_2+|y-y_2|\\
&=\Phi_{c_2}(y).
\end{align*}
Thus $c_1$ reaches every vertex of $K$ with cost no greater than
$c_2$. Any continuation available from a vertex of $K$ after $c_2$
can therefore be reused after $c_1$ with no greater prefix cost.
Hence $c_1$ dominates $c_2$.
\end{proof}

When a new nonterminal state is generated, KIA* compares its profile
with all retained profiles for the same key interval. The new state is
discarded if it is dominated by an existing state; otherwise, it is
retained and removes any existing states that it dominates.

\paragraph{Heuristic Evaluation}
For any $y\in[K.y_s,K.y_e]$, the triangle inequality gives
\[
|y-v_{\mathrm{proj}}.y|
+
\CM((K.x,y),t)
\ge
\CM(v_{\mathrm{proj}},t).
\]
Equality is attained at $y=v_{\mathrm{proj}}.y$. Therefore, the minimum
lower bound over the arrival-cost profile is obtained at the projected
vertex, and we define
\[
h(c)=\CM(v_{\mathrm{proj}},t),
\qquad
f(c)=g(c)+h(c).
\]
Thus $f(c)$ is an admissible lower bound on the cost of any completion
of state $c$.

\paragraph{Terminal State Evaluation}
For a terminal state $c$ with
$W(c)=\langle w_1,\ldots,w_n\rangle$ and $w_n=t$, we define
\[
g(c)
=
\sum_{i=1}^{n-1}\CM(w_i,w_{i+1}),
\qquad
h(c)=0.
\]
Thus $f(c)=g(c)$. By
Lemma~\ref{lem:waypoint-feasibility}, $g(c)$ is the exact length of the
grid path represented by $W(c)$.

\subsection{Path Construction}

Let $W=\langle w_1,\ldots,w_n\rangle$ be the waypoint sequence of the
terminal state, where $w_1=s$ and $w_n=t$. The final grid path $\pi$ is
constructed by replacing each consecutive pair $(w_i,w_{i+1})$ with a
Manhattan-feasible grid segment and concatenating the resulting
segments.

If $w_i$ and $w_{i+1}$ lie in the same vertical interval, the segment
between them is the vertical path within that interval. If they share a
non-key context $C$, the segment is constructed along
$\Gamma_C(I_{w_i},I_{w_{i+1}})$ using the crossing-row procedure from
Lemma~\ref{lem:context-manhattan}, where $I_{w_i}$ and $I_{w_{i+1}}$
are the intervals containing the two waypoints. Otherwise, the segment
follows the interval chains recorded along the corresponding search
branch, using the construction established in
Lemma~\ref{lem:waypoint-feasibility}.

Path construction therefore requires no local search; it follows the
corresponding interval chains and runs in time linear in the length of
the output path.

\section{Theoretical Properties}

This section establishes the completeness and optimality of KIA*.

\begin{theorem}[Completeness]
\label{thm:completeness}
KIA* returns a valid path from $s$ to $t$ if one exists, and reports
failure otherwise.
\end{theorem}
\begin{proof}
If the initial non-key-context test succeeds,
Lemma~\ref{lem:context-manhattan} constructs a valid path from $s$ to
$t$. Otherwise, every returned path is reconstructed from a terminal
state. By Lemma~\ref{lem:waypoint-feasibility}, each consecutive
waypoint pair is Manhattan-feasible. Hence every path
returned by KIA* is valid.

Now suppose that a grid path from $s$ to $t$ exists. Map its vertices to
their containing vertical intervals, remove consecutive duplicates, and
erase cycles from the resulting interval walk. This yields a simple
interval chain $\Lambda=(J_0,\ldots,J_r)$ from $I_s$ to $I_t$. The
chain is used only as a witness of reachability; the path returned by
KIA* need not follow the same grid vertices.

If $\Lambda$ contains no key interval, all its intervals belong to the
same non-key component. Thus $s$ and $t$ share a non-key context, and
the initial test succeeds by Lemma~\ref{lem:context-manhattan}.

Otherwise, let $K_1,\ldots,K_m$ be the key intervals in $\Lambda$, in
traversal order. The first belongs to $\mathcal{K}_s$ and the last to
$\mathcal{K}_t$. This is immediate when the corresponding endpoint
interval is key; otherwise, the key interval lies on the boundary of
the non-key component containing that endpoint.

For each $1\le i<m$, let $\Gamma_i$ be the subchain from $K_i$ to
$K_{i+1}$. All its intermediate intervals are non-key, and $\Gamma_i$
is horizontally monotone. Otherwise, a horizontal reversal at an
intermediate interval would make its preceding and succeeding intervals
distinct direct neighbors on the same side. By
Theorem~\ref{thm:key-point}, that intermediate interval would be key,
contradicting the choice of $K_i$ and $K_{i+1}$ as consecutive key
intervals.

During preprocessing, graph construction starts from every key interval
through each of its direct neighbors. Starting from $K_i$ along
$\Gamma_i$, it follows the unique continuation through the non-key
intermediate intervals and reaches $K_{i+1}$. Hence
$\Gamma_i\in\mathcal{E}_{\mathrm{out}}(K_i)$, and the key-interval graph
contains an edge sequence from an interval in $\mathcal{K}_s$ to one in
$\mathcal{K}_t$.

No candidate on this witness sequence is removed by
immediate-backtracking pruning. At each $K_i$, the incoming and outgoing
intervals are the intervals immediately preceding and following $K_i$
in $\Lambda$. If they were equal, $\Lambda$ would contain a subchain
$(I,K_i,I)$ and would not be simple.

KIA* initializes every interval in $\mathcal{K}_s$ and generates every
outgoing successor that is not an immediate backtrack. The waypoint
updates preserve the required continuation by
Lemmas~\ref{lem:waypoint-feasibility} and
\ref{lem:pruning-safety}. If a witness successor is removed by dominance
pruning, Lemma~\ref{lem:state-dominance} guarantees a retained state
whose arrival-cost profile is no greater at any vertex of the same key
interval. The remaining suffix can therefore be realized with no greater
cost; if its first traversal forms an immediate backtrack, removing that
round trip yields a shorter realization of the same suffix. Thus pruning
cannot eliminate all realizations of the witness sequence.

Finally, the search terminates. There are finitely many key intervals
and finitely many projected rows on each interval. For a fixed key
interval and projected row, every newly retained profile must have a
strictly smaller nonnegative integer $g$-value than the profile it
replaces. Hence only finitely many profile improvements and expansions
can occur.

Therefore, if a path exists, KIA* eventually generates and extracts a
terminal state. Conversely, if no path exists, no terminal state can be
generated because every terminal state represents a valid grid path.
The open list eventually becomes empty, and KIA* reports failure.
\end{proof}

\begin{table*}[t]
\centering
\small
\setlength{\tabcolsep}{2.5pt}
\renewcommand{\arraystretch}{1.08}
\resizebox{\textwidth}{!}{%
\begin{tabular}{|l|c|ccccc|c|cccc|cccc|}
\hline
\textbf{Benchmark}
& \textbf{Maps}
& \multicolumn{5}{c|}{\textbf{Runtime per Instance (ms)}}
& \textbf{Length}
& \multicolumn{4}{c|}{\textbf{Preprocessing Time (ms)}}
& \multicolumn{4}{c|}{\textbf{Retained Memory (MB)}} \\
\cline{3-16}
&
& \textbf{A*}
& \textbf{JPS4}
& \textbf{RSR}
& \textbf{HPA*}
& \textbf{KIA*}
& \textbf{HPA*}
& \textbf{JPS4}
& \textbf{RSR}
& \textbf{HPA*}
& \textbf{KIA*}
& \textbf{JPS4}
& \textbf{RSR}
& \textbf{HPA*}
& \textbf{KIA*} \\
\hline
Random (10\%)
& 10
& 0.348 & 0.217 & 0.394 & 0.126 & 0.450
& 1.011
& 3.682 & 34.811 & 118.847 & 61.689
& 0.10 & 13.29 & 5.39 & 20.35 \\

Random (15\%)
& 10
& 0.571 & 0.358 & 0.621 & 0.164 & 0.662
& 1.012
& 4.070 & 31.789 & 127.273 & 97.842
& 0.10 & 13.25 & 5.77 & 26.57 \\

Random (20\%)
& 10
& 0.760 & 0.496 & 0.852 & 0.214 & 0.829
& 1.011
& 4.488 & 29.713 & 125.663 & 117.380
& 0.10 & 12.59 & 5.82 & 30.51 \\

Random (25\%)
& 10
& 0.979 & 0.670 & 1.157 & 0.272 & 0.961
& 1.012
& 4.878 & 27.116 & 115.659 & 131.660
& 0.10 & 11.92 & 5.60 & 32.59 \\

Random (30\%)
& 10
& 1.216 & 0.846 & 1.410 & 0.327 & 1.018
& 1.015
& 5.207 & 25.390 & 93.492 & 129.795
& 0.10 & 11.25 & 5.18 & 32.05 \\

Random (35\%)
& 10
& 1.794 & 1.218 & 1.935 & 0.408 & 1.110
& 1.018
& 5.384 & 22.991 & 65.932 & 119.878
& 0.10 & 10.53 & 4.62 & 31.17 \\

Random (40\%)
& 10
& 1.903 & 1.293 & 1.955 & 0.406 & 0.882
& 1.017
& 4.247 & 13.065 & 28.554 & 58.559
& 0.10 & 7.48 & 3.17 & 18.35 \\

Maze
& 60
& 4.060 & 0.680 & 2.541 & 0.548 & 0.305
& 1.036
& 2.798 & 11.141 & 26.150 & 34.556
& 0.10 & 7.05 & 3.33 & 9.85 \\

Room
& 40
& 1.739 & 0.139 & 0.760 & 0.242 & 0.051
& 1.020
& 2.443 & 8.163 & 36.148 & 12.051
& 0.10 & 6.33 & 3.60 & 4.44 \\

Street
& 90
& 1.254 & 0.030 & 1.646 & 0.133 & 0.008
& 1.011
& 4.325 & 29.594 & 38.361 & 3.528
& 0.17 & 16.14 & 4.54 & 1.07 \\

BG512
& 75
& 0.485 & 0.019 & 0.391 & 0.068 & 0.004
& 1.011
& 2.115 & 2.984 & 5.933 & 1.269
& 0.10 & 4.26 & 1.64 & 0.39 \\

BGMaps
& 120
& 0.039 & 0.006 & 0.054 & 0.011 & 0.002
& 1.053
& 0.169 & 0.442 & 0.491 & 0.304
& 0.01 & 0.38 & 0.12 & 0.12 \\

Dragon Age
& 223
& 0.189 & 0.019 & 0.218 & 0.035 & 0.006
& 1.039
& 1.108 & 1.543 & 1.942 & 1.006
& 0.06 & 2.22 & 0.75 & 0.33 \\

StarCraft
& 75
& 2.155 & 0.094 & 2.127 & 0.215 & 0.031
& 1.010
& 4.195 & 17.209 & 26.960 & 6.148
& 0.17 & 12.05 & 3.89 & 2.35 \\
\hline
\end{tabular}%
}
\caption{Results on the MovingAI 2D benchmarks, with Random maps
separated by nominal obstacle density. Each Random density contains 10
maps. Runtime is the per-scenario median over five measured runs,
averaged within each map and then equally across maps. HPA* uses cluster
size 10. Length reports the mean per-instance HPA*/A* path-length ratio;
JPS4, RSR, and KIA* match A* on every evaluated scenario. Preprocessing time and retained memory are also averaged with equal
weight across maps, with online query workspace
excluded from retained memory.}
\label{tab:main_results}
\end{table*}

\begin{theorem}[Optimality]
\label{thm:optimality}
Whenever a path from $s$ to $t$ exists, KIA* returns a shortest
4-connected grid path.
\end{theorem}
\begin{proof}
Let $L^\star=\C(s,t)$. If the initial non-key-context test succeeds,
KIA* returns a path of length $\CM(s,t)$. Since Manhattan distance
lower-bounds the length of every 4-connected grid path, this path is
shortest. We therefore consider the case in which interval-level search
is performed.

Among all shortest paths, choose one, denoted by $\pi^\star$, with the
fewest interval transitions. Its interval walk is simple. Otherwise, if
the same vertical interval occurred twice at vertices $a$ and $b$, the
subpath between them could be replaced by the vertical segment within
that interval. This replacement would not increase the path length and
would reduce the number of interval transitions, contradicting the
choice of $\pi^\star$.

As in the proof of Theorem~\ref{thm:completeness}, the interval walk of
$\pi^\star$ decomposes into horizontally monotone subchains between
consecutive key intervals. Each such subchain is an edge of the
key-interval graph, while the initial and terminal portions, when
nonempty, lie in the corresponding non-key components. Thus
$\pi^\star$ induces a sequence of candidate successors from an initial
state in $\mathcal{K}_s$ to a terminal successor through an interval in
$\mathcal{K}_t$. Simplicity of the interval walk ensures that none of
these candidates is an immediate backtrack.

We show that the search retains a state admitting a completion of total
cost at most $L^\star$. At the first key interval, the initialization
profile reaches the vertex used by $\pi^\star$ with cost no greater than
the corresponding prefix cost. This is immediate when $I_s$ is key and
follows from Lemma~\ref{lem:context-manhattan} otherwise.

Suppose that a retained state $c$ at a key interval $K$ admits such a
completion. When $c$ is expanded, consider the successor corresponding
to the next edge of that completion. If no waypoint is appended, the
completion remains available. If a key point or transition vertex is
appended, Lemma~\ref{lem:pruning-safety} allows the local traversal to
be modified to pass through that waypoint without increasing its
length. The generated successor therefore also admits a completion of
cost at most $L^\star$.

If the successor is removed by dominance pruning,
Lemma~\ref{lem:state-dominance} provides a retained state whose
arrival-cost profile is no greater at every vertex of the same key
interval. The remaining suffix can therefore be realized with no greater
total cost; if its first traversal forms an immediate backtrack, removing
that round trip yields a shorter realization of the same suffix. Hence
the retained state also admits a completion of cost at most
$L^\star$.

Consequently, until a terminal state of cost at most $L^\star$ is
extracted, the open list contains either such a terminal state or a
nonterminal state admitting a completion of cost at most $L^\star$.
By the heuristic analysis in State Dominance and Cost Evaluation,
$f(c)$ lower-bounds the cost of every completion represented by a
nonterminal state $c$. Therefore, any such frontier state satisfies
$f(c)\le L^\star$. For a terminal state, $h(c)=0$ and $f(c)=g(c)$,
where $g(c)$ is the exact length of its reconstructed grid path.

At each iteration, KIA* extracts a state of minimum current $f$-value.
Thus, while the open list contains a state with $f\le L^\star$, a
terminal state of cost greater than $L^\star$ cannot be extracted first.
The first extracted terminal state therefore has cost at most
$L^\star$. Since it represents a valid grid path, its cost is also at
least $\C(s,t)=L^\star$. Its cost is therefore exactly $L^\star$, and
KIA* returns a shortest 4-connected grid path.
\end{proof}
\section{Experimental Results}

We evaluate Key-Interval~A* in terms of online runtime,
preprocessing time, retained preprocessing memory, and path quality on
the MovingAI grid benchmark suite~\cite{sturtevant2012benchmarks}. We
compare KIA* with A*, Jump Point Search for 4-connected grids
(JPS4)~\cite{baum2025jps4}, Rectangular Symmetry Reduction
(RSR)~\cite{DBLP:conf/aiide/HaraborB10}, and HPA* with cluster size 10. All solvers use optimized
C++ implementations and were run single-threaded on an AMD Ryzen~7
9700X processor with 16\,GB RAM. Our JPS4 implementation is adapted from the Warthog pathfinding
library~\cite{warthog} and retains its cache-friendly bitsets,
precomputed horizontal jumps, and reusable query workspace. Its
map-dependent preprocessing cost and persistent bitset and jump-table
storage are included in the reported measurements; the reusable query
workspace is excluded from retained memory, consistently with the other
solvers. The source code and experimental scripts are available at
\url{https://github.com/TaiquanSui/Key-Interval-Astar}.

Each map contributes up to 500 valid scenarios selected by a
deterministic bucket-stratified procedure, with the same scenario set
used by all solvers. Each solver performs one warm-up batch followed by
five measured batches. Runtime is the per-scenario median over the five
runs, averaged within each map and then equally across maps.
Preprocessing time measures map-dependent initialization, while retained
memory measures map-dependent preprocessing data retained between
queries and excludes online query workspaces. The Dragon Age group combines the DAO and DA2
map sets, whereas BGMaps is reported separately.

Table~\ref{tab:main_results} summarizes the results. KIA*
achieves the lowest online runtime on all seven non-Random benchmark
groups. Relative to the fastest non-KIA* baseline in each group, its
speedup is approximately \(1.8\times\) on Maze, \(2.7\times\) on Room,
\(3.8\times\) on Street, \(4.7\times\) on BG512, \(3.0\times\) on
BGMaps, \(3.2\times\) on Dragon Age, and \(3.0\times\) on StarCraft.
The largest gains occur on structured and game-derived maps containing
rooms, corridors, and large regular regions.

The density-separated Random results reveal a more nuanced trend. HPA*
has the lowest online runtime at every evaluated density. At 10\% and
15\% obstacle density, KIA* is slower than all four baselines. At 20\%,
it becomes slightly faster than RSR, and at 25\% and 30\% it is faster
than both A* and RSR. At 35\% and 40\%, KIA* also overtakes JPS4 while
remaining slower than HPA*. Thus, KIA* becomes increasingly competitive
online as obstacle density rises, although it does not outperform HPA*
on Random maps.

KIA*'s preprocessing requirements depend strongly on the compactness of
the interval representation. On Street, BG512, BGMaps, Dragon Age, and
StarCraft, retained memory ranges from 0.12\,MB to 2.35\,MB. Random
maps are substantially more expensive and exhibit a non-monotonic
trend: retained memory increases from 20.35\,MB at 10\% obstacle
density to 32.59\,MB at 25\%, while preprocessing time rises from
61.689\,ms to 131.660\,ms. Both quantities subsequently decrease,
reaching 18.35\,MB and 58.559\,ms at 40\%. This pattern is consistent
with an interaction between structural fragmentation and the amount of
traversable space: intermediate densities may induce many intervals and
boundary changes, whereas the highest densities leave less free space
to represent. Maze maps also incur relatively high preprocessing costs,
consistent with their narrow corridor structure. JPS4 retains
approximately 0.10\,MB across the Random densities because these maps
have identical \(512\times512\) dimensions and its persistent bit-grid
storage depends primarily on map dimensions rather than obstacle
density.

KIA*, JPS4, and RSR match the path lengths returned by A* on every
evaluated scenario. HPA* returns slightly longer paths, with mean
per-instance HPA*/A* length ratios ranging from 1.010 on StarCraft to
1.053 on BGMaps.

Overall, KIA* provides its strongest online speedups and lowest
preprocessing overhead when free space admits a compact interval
representation, while preserving exact 4-connected shortest-path
lengths.

\section{Conclusion and Future Work}

We presented Key-Interval A*, an optimal pathfinding algorithm
for 4-connected grids that searches over a compact interval-level
abstraction derived from structural changes in free space. KIA* combines
lightweight preprocessing, A*-style search over a key-interval graph,
and constructive grid-path recovery without cell-level local search. We
established its completeness and optimality, and experiments showed that
KIA* preserves exact shortest-path lengths while achieving strong runtime
performance, particularly on structured and game maps.

Future work includes improving efficiency on highly fragmented random
maps, extending the abstraction to 8-connected and any-angle
pathfinding, and exploiting the preprocessed map structure for efficient
local repair and partial replanning in dynamic environments.

\bibliography{aaai2027}

@article{DBLP:journals/tssc/HartNR68,
  author       = {Peter E. Hart and
                  Nils J. Nilsson and
                  Bertram Raphael},
  title        = {A Formal Basis for the Heuristic Determination of Minimum Cost Paths},
  journal      = {{IEEE} Trans. Syst. Sci. Cybern.},
  volume       = {4},
  number       = {2},
  pages        = {100--107},
  year         = {1968},
  url          = {https://doi.org/10.1109/TSSC.1968.300136},
  doi          = {10.1109/TSSC.1968.300136},
  bibsource    = {dblp computer science bibliography, https://dblp.org}
}

@inproceedings{DBLP:conf/aaai/HaraborG11,
  author       = {Daniel Damir Harabor and
                  Alban Grastien},
  editor       = {Wolfram Burgard and
                  Dan Roth},
  title        = {Online Graph Pruning for Pathfinding On Grid Maps},
  booktitle    = {Proceedings of the Twenty-Fifth {AAAI} Conference on Artificial Intelligence,
                  {AAAI} 2011, San Francisco, California, USA, August 7-11, 2011},
  pages        = {1114--1119},
  publisher    = {{AAAI} Press},
  year         = {2011},
  url          = {https://doi.org/10.1609/aaai.v25i1.7994},
  doi          = {10.1609/AAAI.V25I1.7994},
  bibsource    = {dblp computer science bibliography, https://dblp.org}
}

@inproceedings{DBLP:conf/aiide/HaraborB10,
  author       = {Daniel Harabor and
                  Adi Botea},
  editor       = {G. Michael Youngblood and
                  Vadim Bulitko},
  title        = {Breaking Path Symmetries on 4-Connected Grid Maps},
  booktitle    = {Proceedings of the Sixth {AAAI} Conference on Artificial Intelligence
                  and Interactive Digital Entertainment, {AIIDE} 2010, October 11-13,
                  2010, Stanford, California, {USA}},
  publisher    = {The {AAAI} Press},
  year         = {2010},
  url          = {http://aaai.org/ocs/index.php/AIIDE/AIIDE10/paper/view/2122},
  bibsource    = {dblp computer science bibliography, https://dblp.org}
}

@article{DBLP:journals/jogd/Botea0S04,
  author       = {Adi Botea and
                  Martin M{\"{u}}ller and
                  Jonathan Schaeffer},
  title        = {Near Optimal Hierarchical Path-Finding},
  journal      = {J. Game Dev.},
  volume       = {1},
  number       = {1},
  pages        = {1--30},
  year         = {2004},
  bibsource    = {dblp computer science bibliography, https://dblp.org}
}

@inproceedings{DBLP:conf/aaai/YapBHS11,
  author       = {Peter Yap and
                  Neil Burch and
                  Robert C. Holte and
                  Jonathan Schaeffer},
  editor       = {Wolfram Burgard and
                  Dan Roth},
  title        = {Block A*: Database-Driven Search with Applications in Any-Angle Path-Planning},
  booktitle    = {Proceedings of the Twenty-Fifth {AAAI} Conference on Artificial Intelligence,
                  {AAAI} 2011, San Francisco, California, USA, August 7-11, 2011},
  pages        = {120--125},
  publisher    = {{AAAI} Press},
  year         = {2011},
  url          = {https://doi.org/10.1609/aaai.v25i1.7813},
  doi          = {10.1609/AAAI.V25I1.7813},
  bibsource    = {dblp computer science bibliography, https://dblp.org}
}

@inproceedings{DBLP:conf/aips/UrasKH13,
  author       = {Tansel Uras and
                  Sven Koenig and
                  Carlos Hern{\'{a}}ndez},
  editor       = {Daniel Borrajo and
                  Subbarao Kambhampati and
                  Angelo Oddi and
                  Simone Fratini},
  title        = {Subgoal Graphs for Optimal Pathfinding in Eight-Neighbor Grids},
  booktitle    = {Proceedings of the Twenty-Third International Conference on Automated
                  Planning and Scheduling, {ICAPS} 2013, Rome, Italy, June 10-14, 2013},
  publisher    = {{AAAI}},
  year         = {2013},
  url          = {http://www.aaai.org/ocs/index.php/ICAPS/ICAPS13/paper/view/6058},
  bibsource    = {dblp computer science bibliography, https://dblp.org}
}

@inproceedings{DBLP:conf/aiide/Botea11,
  author       = {Adi Botea},
  editor       = {Vadim Bulitko and
                  Mark O. Riedl},
  title        = {Ultra-Fast Optimal Pathfinding without Runtime Search},
  booktitle    = {Proceedings of the Seventh {AAAI} Conference on Artificial Intelligence
                  and Interactive Digital Entertainment, {AIIDE} 2011, October 10-14,
                  2011, Stanford, California, {USA}},
  publisher    = {The {AAAI} Press},
  year         = {2011},
  url          = {http://www.aaai.org/ocs/index.php/AIIDE/AIIDE11/paper/view/4050},
  bibsource    = {dblp computer science bibliography, https://dblp.org}
}

@article{sturtevant2012benchmarks,
  title={Benchmarks for Grid-Based Pathfinding},
  author={Sturtevant, N.},
  journal={Transactions on Computational Intelligence and AI in Games},
  volume={4},
  number={2},
  pages={144 -- 148},
  year={2012},
  url = {http://web.cs.du.edu/~sturtevant/papers/benchmarks.pdf},
}

@misc{baum2025jps4,
  title        = {Jump Point Search Pathfinding in 4-connected Grids},
  author       = {Baum, Johannes},
  year         = {2025},
  eprint       = {2501.14816},
  archivePrefix = {arXiv},
  primaryClass = {cs.RO},
  doi          = {10.48550/arXiv.2501.14816}
}

@misc{warthog,
  author       = {Harabor, Daniel and Hechenberger, Ryan},
  title        = {Warthog: A Highly-Optimised Pathfinding Search Solver},
  year         = {2024},
  howpublished = {\url{https://bitbucket.org/dharabor/pathfinding/src/master/warthog/}},
  note         = {Software library, accessed 2026-05}
}

\end{document}